\documentclass[runningheads,envcountsame]{llncs}

\usepackage[T1]{fontenc}
\usepackage{graphicx}
\usepackage{mathtools}  
\usepackage{dsfont}
\usepackage{amsmath}
\usepackage{mismath}
\usepackage{pgfplots}
\pgfplotsset{compat=1.7}
\usepackage{amssymb}
\usepackage{cite}
\usepackage[font=footnotesize]{caption}

\usepackage[ruled,vlined,linesnumbered]{algorithm2e}
\SetKwComment{Comment}{/* }{ */}
\SetKwInOut{Input}{Initialization}
\SetKwInOut{Output}{Output}

\newcommand{\rlo}{$r$\nobreakdash-\textsc{LeadingOnes}\xspace}
\newcommand{\rcga}{$r$\nobreakdash-cGA\xspace}
\newcommand{\lo}{\textsc{LeadingOnes}\xspace}
\newcommand{\cga}{cGA\xspace}
\newcommand{\ronemax}{$r$\nobreakdash-OneMax\xspace}
\newcommand{\onemax}{OneMax\xspace}

\newcommand{\ie}{i.\,e.\xspace}
\newcommand{\expect}[1]{\textup{E}[#1]}
\newcommand{\prob}[1]{\textup{P}[#1]}
\newcommand{\indic}[1]{\mathds{1}\{#1\}}

\newenvironment{proofof}[1]{\noindent \textbf{Proof of #1:}}{$\Box$}

\begin{document}

\title{A Runtime Analysis of the Multi-Valued Compact Genetic Algorithm on Generalized \lo} 
\titlerunning{A Runtime Analysis of the Multi-Valued \cga on Generalized \lo}
%
\author{Sumit Adak \and
Carsten Witt}
\authorrunning{S. Adak and C. Witt}
%
\institute{DTU Compute, Technical University of Denmark\\ Kgs. Lyngby, Denmark\\
\email{\{suad,cawi\}@dtu.dk}}
\maketitle              
\begin{abstract}

In the literature on runtime analyses of estimation of distribution algorithms (EDAs), researchers have recently explored univariate EDAs for multi-valued decision variables. Particularly, Jedidia et al.\ gave the first runtime analysis of the multi-valued UMDA on the $r$-valued \lo (\rlo) functions and Adak et al.\ gave the first runtime analysis of the multi-valued \cga (\rcga) on the $r$-valued \onemax function. We utilize their framework to conduct an analysis of the multi-valued \cga on the $r$-valued \lo function. Even for the binary case, a runtime analysis of the classical \cga on \lo was not yet available. In this work, we show that the runtime of the \rcga on \rlo is $\bigo(n^2 r^2\log^3 n\log^2 r)$ with high probability.

\end{abstract}
\section{Introduction}
\label{section:introduction}

An optimization technique known as \emph{estimation of distribution algorithms} (EDAs) builds a probabilistic model that is subsequently used to generate new search points based on earlier searches. Three main phases are involved when creating an EDA; by using the existing probabilistic model, a population of individuals is first sampled; next, the population's fitness is ascertained; and then, a new probabilistic model is generated based on the population's fitness. The main difference between them and evolutionary algorithms (EAs) is that the latter evolve a population, whilst the former evolve a probabilistic model. A number of studies \cite{benbaki2021rigorous,doerr2021runtime,witt2023majority} have shown that EDAs can outperform EAs. 

Different probabilistic models and update techniques give rise to distinct algorithms within EDAs. Each approaches offers unique advantages and challenges, making it suitable for different types of optimization problems. According to the strength of the underlying probabilistic model, EDAs can be categorized as \emph{univariate} and \emph{multivariate} algorithms. Only one variable is used in the model of each problem variable by univariate algorithms, and, on the other hand, multivariate algorithms employ multiple variables to model a problem variable. A more detailed classification of EDAs is provided by Pelikan et al.\ \cite{pelikan2015estimation}. A couple of popular examples of multivariate EDAs include the \emph{factorized distribution algorithm} (FDA)~\cite{Muhlenbein1999}, the \emph{extended compact genetic algorithm} (ecGA)~\cite{harik2006linkage}, the \emph{mutual-information-maximization input clustering} (MIMIC)~\cite{de1996mimic}, the \emph{bivariate marginal distribution algorithm} (BMDA)~\cite{pelikan1999bivariate}, and the \emph{Bayesian optimization algorithm} (BOA)~\cite{Pelikan-BOA}. Univariate EDAs examples include the \emph{univariate marginal distribution algorithm} (UMDA) \cite{muhlenbein1996recombination}, the \emph{population-based incremental learning} (PBIL), \cite{baluja1994population}, the \emph{compact genetic algorithm} (\cga) \cite{harik1999compact}, and many more. This manuscript is devoted to the theoretical investigations of univariate EDAs, particularly to the multi-valued compact genetic algorithm introduced in \cite{BENJEDIDIA2024114622}.

\onemax \cite{Muhlenbein92} and \lo \cite{rudolph1997convergence} are the pseudo-Boolean functions for EDAs that are most frequently examined theoretically. Other functions have been examined as well, the probably most well-known of which is \textsc{BinVal} \cite{Muhlenbein1999}. Traditional evolutionary algorithms are frequently utilized for diverse search spaces, while EDAs are typically applied for problems that involve binary decision variables. Moreover, researchers made the first moves toward utilizing EDAs in scenarios where decision variables have more than two values \cite{Roberto2008,BENJEDIDIA2024114622,AdakPPSN2024}. Specially, Jedidia et al.\ \cite{BENJEDIDIA2024114622} and Adak et al.\ \cite{AdakPPSN2024} explore univariate EDAs for multi-valued decision variables. By adding $r$ probability values for each variable, they address a multi-valued problem. This article addresses multi-valued \lo.

In the literature on runtime analyses of EDAs, there are several analyses of the UMDA on \lo and variants \cite{Chen2009,Chen2010}. But a runtime analysis of even the simple binary \cga on \lo was to date missing. Hence, one of our goals is to provide a first runtime analysis of the \cga on \lo.
In fact, in this article, we address a more general case by providing the first runtime analysis of the multi-valued \cga (\rcga) on the $r$-valued \lo (\rlo) function, giving additional insight on the performance of multi-valued EDAs.
Specifically, we bound the runtime of the \rcga on the \rlo problem by $\bigo(n^2 r^2\log^3 n\log^2 r)$ (Theorem~\ref{theorem:rcgacomplexity}) with high probability (defined 
as probability $1-o(1)$). For $r=\bigo(1)$, this 
bound is close to the typical runtime around $\Theta(n^2)$ that simple EAs \cite{DROSTE200251} and other EDAs with 
optimal parameter settings \cite{DOERR2021121} have on this problem.

The manuscript is organized as follows: Section~\ref{section:background} defines the multi-valued \lo function and summarizes the earlier work on our technical domains. Section~\ref{section:cGA} elaborates on the multi-valued EDA framework for the multi-valued \cga. The main technical results including genetic drift analysis and the runtime execution of the \rcga on $r$-valued \lo function are presented in Section~\ref{section:drift} and~\ref{section:runtime}. For the \emph{hypothetical population size} $K$, which is the key parameter of the \rcga, the experiments in Section~\ref{section:experiments} present the empirical runtime across the parameter. Finally, the manuscript concludes with a brief summary. 

\section{Background}
\label{section:background}
\subsection{Preliminaries}
\label{subsection:preliminaries}

We focus on the maximization of functions of the form $f:\{0,1,\dots,r-1\}^{n}\rightarrow \mathbb{R}$, also called $r$-valued (or multi-valued) fitness functions.
For an individual $x \in \{0,1,\dots,r-1\}^{n}$, we call $f(x)$ the \emph{fitness} of $x$.

Let $n\in \mathbb{N}_{\geq 1}$ and $r\in \mathbb{N}_{\geq 2}$. We state the definition of \rlo as already defined in \cite{BENJEDIDIA2024114622}. For all $x = (x_1, \dots, x_n) \in \{0,1,\dots,r-1\}^{n}$, we have 
\begin{align*}
r\text{-}\lo(x) \coloneqq &  \sum_{i=1}^{n} \prod^{i}_{j=1}  \indic {x_{j} = 0}
\end{align*}
and the function returns the number of consecutive~$0$s starting from the leftmost position. Note that the unique maximum is the all-$0$s string in \rlo function. However, a more general version can be defined by choosing an arbitrary optimum $a\in \{0,\dots,r-1\}^{n}$, and defining, for all $b\in \{0,\dots,r-1\}^{n}$, $r$-$\lo_{a,\sigma}(b) =  \sum_{i=1}^{n} \prod^{i}_{j=1} \indic{b_\sigma(j) = a_\sigma(j)}$, where $\sigma$ is a permutation of $\{1,\dots,n\}$ \cite{afshani2013query}. Note that, for \rlo, the maximum fitness value is $n$.

\subsection{Related Work}
\label{subsection:RW}

In this work, we concentrate on the runtime evaluation of $r$-valued compact genetic algorithm (\rcga) on the multi-valued \lo function. 
There are many theoretical articles starting from classical evolutionary algorithms to EDAs.
EDAs are commonly utilized to address a wide range of complex problems, as highlighted in recent studies~\cite{droste2006rigorous,sudholt2019choice}. Droste presented the first runtime analysis of the \cga on linear pseudo-Boolean functions~\cite{droste2006rigorous}. It was also shown that the expected runtime for any function has a lower bound of $\Omega(K\sqrt{n})$ and an upper bound of $O(Kn)$ for every linear function. Furthermore, it was noted that the difference in runtime between two linear functions suggests that EDAs optimize problems within this class in distinct ways. Most theoretical research on EDAs has focused on pseudo-Boolean optimization~\cite{Krejca2020}. Among the most commonly studied pseudo-Boolean functions for EDAs are \onemax and \lo~~\cite{Motwani_Raghavan_1995,Muhlenbein92}. In addition to these, \textsc{BinVal} is another widely recognized function that has been explored~\cite{Chen2009}, although other functions have also been analyzed~\cite{droste2006rigorous,Muhlenbein1999}. A framework for EDAs for optimizing problems with more than two choice variables from the multi-valued domain was recently introduced by Jedidia et al.\ \cite{BENJEDIDIA2024114622}. They demonstrate how the multi-valued UMDA effectively solves the $r$-valued \lo problem. Subsequently, Adak et al.\ \cite{AdakPPSN2024} provide the first runtime analysis of a $r$-valued \onemax function using the multi-valued \cga within their framework. Together, their work demonstrates how EDAs can be tailored for multi-valued problems and used to define their parameters. 

The analysis of EDAs for complex problems is a particularly active field of research nowadays~\cite{Krejca2020}. Very recently,  Hamano et al.\ \cite{Ryoki24EC} explored a probabilistic model-based technique with a sample size of two and an underlying distribution derived from the family of categorical distributions, which they termed categorical compact genetic algorithm (ccGA). It turns out that this algorithm is equivalent to the \rcga investigated in this paper. Theoretically, they have investigated the dependency of the number of dimensions, the number of possible categories, and the learning rate on the runtime. In the categorical domain, they have explored the tail bound of the runtime on two linear functions: categorical \onemax (COM), which is equivalent 
to the \ronemax function mentioned earlier, and \textsc{KVal}, an extension of the \textsc{BinVal} function. Furthermore, more information regarding the theory and application of EDAs can be found in~\cite{Krejca2020,larranaga2001estimation,pelikan2015estimation}.

\section{The Multi-valued \cga}
\label{section:cGA}

The compact genetic algorithm (cGA) \cite{harik1999compact} is one of the most popular univariate EDAs. It has only one parameter $K\in\mathbb{R}_{>0}$, which is called hypothetical population size \cite{doerr2021runtime} and it maintains a vector of probabilities (called frequencies). In each iteration of the \cga, it creates two solutions independently. Further, by comparing the fitness values of the two solutions, each frequency is updated (increases or decreases) by $1/K$ in the direction of the better offspring. 

An extended version of \cga is the \rcga where it supports the multi-valued variables instead of binary only~\cite{BENJEDIDIA2024114622}. Algorithm~\ref{algorithm:r-cGA-rOneMax} defines the \rcga. It uses marginal probabilities (denoted as frequencies) $p^{(t)}_{i,j}$ corresponding to the probability at time $t$ of position $i$ and value $j$. Further, in each iteration it creates two solutions $x$ and $y$ independently. After comparing the fitness values of $x$ and~$y$, it updates the frequencies by $\pm 1/K$ in the direction of the better offspring. Note that $K$ indicates the strength of the update of the probabilistic model.

\begin{algorithm}[h]
\caption{$r$-valued Compact Genetic Algorithm ($r$-cGA) for the maximization of $f : \{0,\dots,r-1\}^n \rightarrow \mathbb{R}$}
\label{algorithm:r-cGA-rOneMax}
\Input{$t \gets 0$ \hspace{18em} $p^{(t)}_{i,0} \gets p^{(t)}_{i,1} \gets  p^{(t)}_{i,2} \dots \gets p^{(t)}_{i,r-1} \gets \frac{1}{r}$ where $i\in \{1, 2, \dots, n\}$}
\While{termination criterion not met}{
\For{$i\in \{1, 2, \dots, n\}$ independently}{
$x_{i} \gets j$ with probability $p^{(t)}_{i,j}$ w.r.t. $j=0,\dots,r-1$
\\
$y_{i} \gets j$ with probability $p^{(t)}_{i,j}$ w.r.t. $j=0,\dots,r-1$
\\
}
\If{$f(x) < f(y)$}{ swap $x$ and $y$
}
\For{$i\in \{1, 2, \dots, n\}$}{
\For{$j\in \{0, 1, \dots, r-1\}$}{
$\overline{p}^{(t+1)}_{i,j} \gets p^{(t)}_{i,j} + \frac{1}{K} (\indic{x_{i} = j} - \indic{y_{i}=j})$ \\
$p^{(t+1)}_{i,j}\gets$ restrict $\overline{p}^{(t+1)}_{i,j}$ to be within $[\frac{1}{(r-1)n}, 1-\frac{1}{n}]$ (see \cite{BENJEDIDIA2024114622})}
}
$t\gets t + 1$
}
\end{algorithm}
More precisely, the probabilistic model of the \rcga is defined by an $n\times r$ matrix (the frequency matrix), where each row $i\in\{1,\dots,n\}$ forms a vector $p_i\coloneqq (p^{(t)}_{i,j})_{j\in \{0,\dots,r-1\}}$ (the frequency vector at position $i$). In the frequency matrix, initially each frequency is set to $1/r$, leading to a uniform distribution when sampling the first individuals. We create two individuals $x,y\in \{0,\dots,r-1\}^{n}$. Then, for all $i\in \{1,\dots,n\}$ and all $j \in \{0,\dots,r-1\}$, the probability that $x_i$ and $y_i$ has value $j$ is $p^{(t)}_{i,j}$. By comparing the fitness values of $x$ and $y$, we update the frequency by $1/K$. After updating the frequencies and before restricting them to an interval (see next paragraph), each frequency vector sums to 1 in this model, because exactly one frequency is increased by $1/K$ and exactly one frequency is decreased by this same amount. 

In order to avoid the fixation at 0 or 1, the framework introduced in~\cite{BENJEDIDIA2024114622} restricts 
all frequencies to the interval $[1/((r-1)n), 1-1/n]$; see the paper for details. Note that the restriction procedure may also update frequencies belonging to values in~$\{0,\dots,r-1\}$ that were not sampled in any of the two individuals. The restriction ensures that there is always a positive probability to sample an individual of optimum value; on the negative side, the more complicated update mechanism for $r\ge 3$ rules out the 
so-called well-behaved frequency assumption \cite{sudholt2019choice} that has been useful for the binary \cga. In the rest of the paper, we denote $1/((r-1)n)$ as lower border and $1-1/n$ as upper border.
We are interested in the number of function evaluations that are needed to sample a solution of optimum value. This is proportional to the value of $t$ in the algorithm. Further, this number is referred to as \emph{runtime} or \emph{optimization time}.

To analyze the runtime, we define the concept of \emph{critical position} (introduced in \cite{DOERR2021121} for the binary \cga) according to \rlo. Informally, a position is called critical if all the frequencies for value~$0$ of lower position (left of the current position) have gained the maximum (upper border) probability. Formally, a position $i\in \{1,\dots, n\}$ is called critical if and only if the frequencies $p^{(t)}_{j,0}$ have never been greater than $1-1/n$ at any point in the past where $j\in\{1,\dots,i-1\}$, and the frequency $p^{(t)}_{i,0}$ is less than $1-1/n$. In this paper, we define the index of critical position at time~$t$ by $m_{t}$ ($t\ge 0$). Obviously, $m_{t}$ is non-decreasing over time. A major part of our analysis will deal with bounding the time until $m_{t}$ increases by at least~$1$.

\section{Genetic Drift}
\label{section:drift}

In EDAs, genetic drift is the result of random fluctuations brought on by the process's stochasticity rather than a clear signal from the goal function that would cause a frequency to reach the extreme values. Researchers have examined genetic drift in EDAs in detail in a number of runtime analyses \cite{doerr2020univariate,droste2005not,lengler2021complex,sudholt2019choice,witt2018domino,witt2019upper}, as well as in the works of Shapiro \cite{shapiro2002sensitivity,shapiro2005drift,shapiro2006diversity}. Given the significance of having a solid grasp of genetic drift, we now apply the framework from \cite{Doerr2020ITEV} and build on the insights from \cite{BENJEDIDIA2024114622,AdakPPSN2024} to study genetic drift specifically for the $r$-cGA. 

 In this section, we will prove an upper bound on the effect of genetic drift for \rcga in a similar fashion as Ref.~\cite{Doerr2020ITEV,BENJEDIDIA2024114622,AdakPPSN2024}. This allows us to determine the parameter values for EDAs that avoid the usually unwanted effect of genetic drift. In the following section, we discuss genetic drift and prove a concentration result for \emph{neutral positions}. An upper bound for positions with \emph{weak preference} is also included. Next, we first describe the stochastic processes underlying the probabilistic model in the algorithm.

\subsection{Behavior of the Probabilistic Model}

We look in detail into how the \rcga optimized \rlo and define the change in frequency in one step as $\Delta_{i,j}\coloneqq \Delta^{t}_{i,j}\coloneqq p^{(t+1)}_{i,j} - p^{(t)}_{i,j}$, where
 $i\in\{1,\dots,n\}$ and 
 $j\in\{0,\dots,r-1\}$. Particularly, we are interested in $\Delta_{i,0}$ which is crucial to find the optimum on \rlo. The decision to update the frequency in the next step depends on the strings $x$ and $y$ sampled at current time. Specially, we inspect the effect of a particular position in the \rlo values. To find this, we calculate the fitness of strings $x$ and $y$ up to position $i$ $(i>1)$ as $r\text{-LO}_{i}(x)\coloneqq  \sum_{k=1}^{i} \prod_{j=1}^{k} \indic{x_j=0}$ and $r\text{-LO}_{i}(y)\coloneqq  \sum_{k=1}^{i} \prod_{j=1}^{k} \indic{y_j=0}$. Then, \rlo experiences two kinds of steps which are called \emph{random-walk step (rw-step)} and \emph{biased step (b-step)} \cite{sudholt2019choice} depending on the value of position $i$ and lead to an increase or decrease of frequencies. In the rest of this whole section, we temporarily ignore the corrections made to frequencies after clamping them to the interval $[\tfrac{1}{(r-1)n},1-\tfrac{1}{n}]$. A closer analysis of the update scheme of the \rcga, as defined in \cite{BENJEDIDIA2024114622}, reveals that this correction will decrease the value of a frequency by no more than $\tfrac{1}{(n-1)(r-1)}$ compared to the case without borders. Also, for such a correction to happen, a value whose frequency 
 equals the lower border $\frac{1}{(r-1)n}$ has to be sampled. By a union bound over at most~$r-1$ such values, this happens with probability at most $1/n$, so the expected negative correction of frequencies after clamping is at most $\frac{1}{(r-1)(n-1)^2}$. This term will be considered in our drift analyses in 
 Section~\ref{section:runtime}.

Starting from the left-most position (position~1), if the lowest position which is sampled differently in the two individuals $x$ and $y$ is less than $i$ ($i>1$), then position $i$ is not relevant for the ranking of the samples and it performs a random-walk  step. Formally, if the position $i > \min \{r\text{-LO}_{i}(x),r\text{-LO}_{i}(y)\}$, then the value of position $i$ has no impact on the decision to update with respect to string $x$ or $y$. Hence, the frequency will be increased or decreased by $1/K$ 
with identical probability, which means $p^{(t+1)}_{i,0} = p^{(t)}_{i,0} \pm 1/K$ with probability $p^{(t)}_{i,0} (1-p^{(t)}_{i,0})$. Otherwise, it keeps the same value as previously  $p^{(t+1)}_{i,0}= p^{(t)}_{i,0}$ with the remaining probability. 

A biased step at position $i$ occurs under the following condition. The positions left of $i$ are sampled as all~0 in the two individuals $x$ and $y$. That means the fitness of $x$ and $y$, restricted to the first $i-1$ positions, is $r\text{-LO}_{i-1}(x)= r\text{-LO}_{i-1}(y)= i-1$. In that case, if $x_i=0$ and $y_i\neq 0$ (or $x_i\neq 0$ and $y_i= 0$), then position $i$ determines the decision whether to update with respect to string $x$ or $y$. Hence, both the events of sampling the position $i$ increase the frequency of value 0. This scenario is called a biased step where the section between $x$ and $y$ yields a bias towards increasing the frequency of $p^{(t)}_{i,0}$. So, the frequency is updated by $p^{(t+1)}_{i,0} = p^{(t)}_{i,0} + 1/K$ with  probability $2p^{(t)}_{i,0} (1-p^{(t)}_{i,0})$. Otherwise, the frequency keeps the same value as $p^{(t+1)}_{i,0}= p^{(t)}_{i,0}$ with the remaining probability.   

To start our runtime analysis, we will analyze the stochastic process of the frequency for value 0 at the first position. We analyze the growth of the frequency $p^{(t)}_{1,0}$ from its starting value of $1/r$ to its maximum of $1-1/n$. The value of $p^{(t)}_{1,0}$ changes if the two samples of \rcga differ in the first position in the following sense: one individual samples as 0 and the another one samples as non-zero. Since we are considering only the first position on \rlo, this means that the individual sampling the 0 has the higher fitness value and the frequency of $p^{(t)}_{1,0}$ increases. Note that position~1 is the only position where only biased steps occur, and for all following positions, rw-steps can occur. 

\subsection{Analysis of Genetic Drift for the \rcga}

Genetic drift is usually studied according to the behavior of a \textit{neutral} position of a fitness function. Let $f$ be an $r$-valued fitness function. We call a position $i\in \{1,\dots,n\}$ \textit{neutral} (w.r.t.\ to $f$) if and only if, for all $x\in \{0,\dots,r-1\}^{n}$, and the value of $x_i$ has no influence on the value of $f$. More formally, a position $i$ is neutral if, for all individuals $x, x' \in \{0, \dots, r-1\}^{n}$, whenever $x_j = x'_j$ for all $j \in \{1, \dots, n\} \setminus {i}$, it holds that $f(x) = f(x')$. A  greater portion of this section follows closely  \cite{AdakPPSN2024} and is further adjusted to the present paper. 

In the analysis of genetic drift, an important property of neutral variables is that their frequencies in typical EDAs without margins form martingales \cite{Doerr2020ITEV}. This observation applies to EDAs for the binary representations. Further, this statement extends to the \rcga~\cite[Lemma 1]{AdakPPSN2024}, where they proved that all the frequencies belonging to neutral positions are martingales.

In \cite{BENJEDIDIA2024114622,AdakPPSN2024}, all frequencies of an EDA start at a value $1/r$. They analyze the progress
of the expected value of the frequency and tolerate smaller deviations of the actual frequency value from this expected value up to $1/(2r)$ in either direction. In this article, for the $r$-cGA we follow the same frequencies setting starting from $1/r$ and tolerate a deviation up to $1/(2r)$ in either direction.

Here, we apply a martingale concentration result \cite[Theorem 3.15]{mcdiarmid1998concentration} which allows to exploit the lower sampling variance present at frequencies in $\Theta (1/r)$. We restate an adjusted version of a theorem by McDiarmid \cite[eq. (41)]{mcdiarmid1998concentration}, which was used by Doerr and Zheng~\cite{Doerr2020ITEV} and by Jedidia et al.\ \cite{BENJEDIDIA2024114622}. The following theorem is adopted from Ref.~\cite{AdakPPSN2024}.

\begin{theorem}
\label{theorem:Hoeffding-Azuma-inequality}   
Let $a_1,\dots,a_m\in \mathbb{R}$, and $X_1,\dots,X_m$ be a martingale difference sequence with $\lvert X_k \rvert\leq a_k$ for each $k$. Then for all $\varepsilon \in \mathbb{R}_{\geq 0}$, it holds that
\[\Pr\left [\max_{k=1,\dots,m} \left\vert \sum_{i=1}^{k} X_{i}  \right\vert \geq \varepsilon\right ] \leq 2\exp{\left (-\frac{\varepsilon^{2}}{2\sum_{i=1}^{m} a^{2}_{i}}\right ).}\]
\end{theorem}

Next, we show for how long the frequencies of the $r$-cGA at neutral positions stay concentrated around the initial value of $1/r$ by using Theorem~\ref{theorem:Hoeffding-Azuma-inequality}.

\begin{theorem}
\label{theorem:neutral-frequency-stay}    
Let $f$ be an $r$-valued fitness function with a neutral position $i\in\{1,\dots,n\}$. Consider the $r$-cGA optimizing $f$ with population size $K$. Then, for $j\in \{0,\dots,r-1\}$ and $T\in\mathbb{N}$, we have
\[\Pr\left [\max_{t\in\{0,\dots,T\}} \left\vert p^{(t)}_{i,j} - p^{(0)}_{i,j}  \right\vert \geq \frac{1}{2r}\right ] \leq 2\exp{\left (-\frac{K^2}{8Tr^2}\right ).}\]
\end{theorem}

\begin{proof}
We apply a similar proof strategy as in the proof of \cite[Theorem 2]{Doerr2020ITEV}. For the $r$-cGA, we have sequence of frequencies $(p^{(t)}_{i,j})_{t\in\mathbb{N}}$. For all $j\in \{0,\dots, r-1\}$, we obtain
\begin{align*}
\P[p^{(t+1)}_{i,j} = p^{(t)}_{i,j}+1/K \mid p^{(1)}_{i,j},\dots, p^{(t)}_{i,j}] & = p^{(t)}_{i,j} (1 - p^{(t)}_{i,j})\\
\P[p^{(t+1)}_{i,j} = p^{(t)}_{i,j}-1/K \mid p^{(1)}_{i,j},\dots, p^{(t)}_{i,j}] & = p^{(t)}_{i,j} (1 - p^{(t)}_{i,j}) \\
\P[p^{(t+1)}_{i,j} = p^{(t)}_{i,j} \mid p^{(1)}_{i,j},\dots, p^{(t)}_{i,j}] & = 1 - 2p^{(t)}_{i,j} (1 - p^{(t)}_{i,j})
\end{align*}
From  \cite[Lemma~1]{AdakPPSN2024}, we have $\E(p^{(t+1)}_{i,j} \mid p^{(0)}_{i,j},\dots, p^{(t)}_{i,j})=p^{(t)}_{i,j}$. Consider the martingale difference sequence $R_{t} \coloneqq p^{(t)}_{i,j} - p^{(t-1)}_{i,j}$ where $t \geq 1$ and $p^{(0)}_{i,j} = 1/r$, which satisfies $\lvert R_{t} \rvert \leq 1/K$. Further, by expanding it
\begin{align*}
p_{i,j}^{(k)} & = p_{i,j}^{(0)} + (p_{i,j}^{(1)}-p_{i,j}^{(0)}) + (p_{i,j}^{(2)}-p_{i,j}^{(1)}) + \dots + \\ & \indent (p_{i,j}^{(k)}-p_{i,j}^{(k-1)}) \\
& = p_{i,j}^{(0)} + R_{1} + \dots + R_{k}\\
& = p_{i,j}^{(0)} + \sum_{\ell=1}^{k} R_{\ell}\\
\text{so, } p_{i,j}^{(k)} - \frac{1}{r} & = \sum_{\ell=1}^{k} R_{\ell}
\end{align*}

By the \textit{Hoeffding-Azuma inequality} (Theorem~\ref{theorem:Hoeffding-Azuma-inequality}), we have 
\begin{align*}
\Pr\left [\max_{k=1,\dots,T} \left\vert p^{(k)}_{i,j} - \frac{1}{r}  \right\vert \geq \frac{1}{2r} \right] & \\
\Pr\left [\max_{k=1,\dots,T} \left\vert \sum_{\ell=1}^{k} R_{\ell} \right\vert \geq \frac{1}{2r} \right] & \leq 2\exp{\left (-\frac{K^2}{8Tr^2}\right )}.\quad \qed
\end{align*}
\end{proof}

In many situations, positions are not neutral for a given fitness function. However, we prove that the results on neutral positions translate to positions where one value is better compared to all other values. This is referred to as \textit{weak preference}~\cite{Doerr2020ITEV}. Formally, we can say that an $r$-valued fitness function $f$ has a weak preference for a value $j\in \{0,\dots,r-1\}$ at a position $i\in \{1,\dots,n\}$, if and only if, for all $x_1,\dots,x_n \in \{0,\dots,r-1\}$, it holds that 
\[ f(x_1,\dots,x_{i-1},x_{i},x_{i+1},\dots,x_n) \leq f(x_1,\dots,x_{i-1},j,x_{i+1},\dots,x_n).\]

Now, by applying Theorem~3 by Adak and Witt~\cite[Theorem 3]{AdakPPSN2024}, we can extend Theorem~\ref{theorem:neutral-frequency-stay} to positions with weak preference.

\begin{theorem}
\label{theorem:frequency-weak-preference}  
Let $f$ be an $r$-valued fitness function with a weak preference for 0 at position $i\in \{1,\dots,n\}$. Consider the $r$-cGA optimizing $f$ with parameter $K$. Let $T\in \mathbb{N}$, then we have 
\[\Pr\left [\min_{t\in\{0,\dots,T\}} p^{(t)}_{i,0} \leq p_{i,0}^{(0)} - \frac{1}{2r}\right ] \leq \mathord{2\exp}\mathord{{\left (-\frac{K^2}{8Tr^2}\right ).}}\]
\end{theorem}

\begin{proof}
Let $g$ be an $r$-valued fitness function with neutral position $i\in\{1,\dots,n\}$ and frequency matrix $q$. Consider the $r$-cGA optimizing $g$. According to \cite[Theorem~3]{AdakPPSN2024}, $p^{(t)}_{i,0}$ stochastically dominates $q^{(t)}_{i,0}$ for all $t\in\mathbb{N}$. By applying Theorem~\ref{theorem:neutral-frequency-stay} to fitness function $g$ for position $i$, we have
\[\Pr\left [\min_{t\in\{0,\dots,T\}} q^{(t)}_{i,0} \leq \frac{1}{2r}\right ] \leq \mathord{2\exp}\mathord{\left (-\frac{K^2}{8Tr^2}\right )}\]
Using the stochastic domination yields the tail bound for $f$. \qed
\end{proof}

\section{Runtime Analysis}
\label{section:runtime}

In this section, we present the runtime results of the \rcga (Algorithm~\ref{algorithm:r-cGA-rOneMax}) on \rlo. To prove our results, we have used the methods of \textit{occupation probabilities} as given in Lemma \ref{lemma:occupation} \cite{KLWFOGA15}. The concept is that the frequency (for sampling a zero) stays close to a so-called target state~0 after having been there once, where state~0 corresponds to frequency value $1-1/n$. To prove this, we will relate the frequency value to a Markov process $X_t$ ($t\ge 0$) on $\mathbb{R}$ with a drift towards 0. In particular, the analysis will exploit that the occupation probability of state exactly~0 can only be reduced slightly in one iteration of the \rcga.

Similar to previous works, we will derive bounds on  the runtime of the \rcga that hold with high probability (\ie, with probability $1-o(1)$). Bounding the expected runtime presents additional challenges, which could be addressed in future work. The runtime analysis crucially depends on the event that no frequency drops below $1/(2r)$ by genetic drift. Our main runtime result is formulated in the following theorem.

\begin{theorem}
\label{theorem:rcgacomplexity}
With high probability, the runtime of the \rcga on the function \rlo with $K\geq c n r^2 \log^2 n\log r$ for a sufficiently large $c > 0$ and $K=o(n^2)$, $r= poly(n)$ is $\bigo(nK\log r\log K)$. For $K = c n r^2 \log^2 n\log r$, the bound is $\bigo(n^2 r^2\log^3 n\log^2 r)$. 
\end{theorem}

To prove the above theorem, we need the following lemmas dealing with the above-mentioned occupation probabilities. The following lemma bounds the probability of a frequency being at least~$1-2/n$ under the assumption that all frequencies left of it satisfy this bound. While analyzing how a frequency approaches its maximum, we stop our  analysis at the first point in time where it has become at least $1-1/n-1/K$. This ensures that the drift bounds derived in the following are not affected by capping a frequency at its upper border. By our assumption on~$K$ from Theorem~\ref{theorem:rcgacomplexity}, we have $1-1/n-1/K=1-1/n-o(1/n)$, which does not change the 
asymptotic result.

\begin{lemma}
\label{lemma:rediscover-prefix}
Let $K\ge cn\ln n$ for a sufficiently large constant~$c>0$ and $K=o(n^2)$. Consider an index~$i\in\{2,\dots,n\}$ and a time~$t^*\ge 0$ such that $p_{i,0}^{(t^*)}=1-1/n$. For a period of length $T>0$, assume that $p_{j,0}^{(t)}\ge 1-2/n$ for all $t\in[t^*,t^*+T]$ and all $j<i$. Then for all $t\in[t^*,t^*+T]$, it holds that $p_{i,0}^{(t)}\ge 1-2/n$ with probability at least $1-2T e^{1-c/(6e^4)\ln n}$. 
\end{lemma}

This lemma implies the following corollary via a straightforward union bound over at most $n$ positions.

\begin{corollary}
\label{corollary:positionbound}
Let $K\ge 6e^4 c n\ln n$ for a constant $c>0$, and $K=o(n^2)$ and $T\in\N^+$. Then, for any point in time $t\in[0,T]$, all frequencies for value~$0$ left of the current critical position are bounded from below by $1-2/n$ with probability at least $1-O(Tn^{-c+1})$.
\end{corollary}

To show our lemma, we will need the following helper result from \cite{KLWFOGA15}. Therein, ``additive drift at least~$d$ towards~$0$'' means that $\expect{X_t-X_{t+1} \mid X_t}\ge d$ for all $X_t>0$, ``step size at most~$c$'' means $\lvert X_t-X_{t+1}\rvert\le c$ with probability~$1$  for all $t\ge 0$, and ``self-loop probability at least~$p_0$'' means $\prob{X_{t+1}=X_t \mid X_t}\ge p_0$ for all $t\ge 0$. 

\begin{lemma}[Theorem~7 in \cite{KLWFOGA15}]
\label{lemma:occupation}
Let a Markov process $X_t$, $t\ge 0$, on $\R^+_0$ with additive drift at least $d$ towards $0$ be given, starting at~$0$ (\ie, $X_0=0)$, with step size at most~$c$ and self-loop probability at least $p_0$, Then we have for all $t\in\N$ and $b\in \R_0^+$ that 
\[\prob{X_t\ge b} \le 2e^{\frac{2d}{3c(1-p_0)}(1-b/c)}.\]
\end{lemma}

\begin{proofof}{Lemma~\ref{lemma:rediscover-prefix}}
This proof is crucially based on an application of Lemma~\ref{lemma:occupation}. We consider the stochastic process on~$p_{i,0}^{(t)}$, \ie, the frequency of position~$i$ at value~0, and measure its distance from the upper border in units of $1/K$-steps; more formally, let $X_t = K(1-1/n - p_{i,0}^{(t)})$. For notational convenience, we assume $t^*=0$ and obtain $X_0=0$ from our assumption  that $p_{i,0}^{(t^*)}=1-1/n$. We will set $b=K/n$, corresponding to a frequency of $1-2/n$, and analyze the probability of the event $E_t\coloneqq X_t \ge b$ for $t\ge 0$. As long as $X_t\le b$, the probability of updating the frequency is bounded from above by $(4/n)(1-2/n)$ using a union bound since in one of the two individuals, position~$i$ must be sampled as~$0$ and differently in the other individual. Hence, we work with a self-loop probability of at least $p_0\ge 1-4/n$ before the first occurrence of $E_t$. Clearly, we have $c=1$ as a bound on the step size of the scaled process.

To analyze the drift of the $X_t$-process, we again distinguish between random-walk and biased steps. Let $B_t$ be the event that all positions left of~$i$ are sampled as~$0$ in both individuals sampled by the \rcga at time~$t$. Under~$B_t$, a biased step occurs, so frequency $i$ cannot decrease and it increases with probability $q_t\coloneqq 2p_{i,0}^{(t)}(1-p_{i,0}^{(t)})$ unless it is at the upper border already. If~$B_t$ does not occur, a random-walk step occurs and each with probability~$q_t$, the frequency increases and decreases (again up to hitting a border). Hence, if $X_t>0$, then an increasing step happens with probability at least $(1+\prob{B_t})q_t/2$ and a decreasing step with probability $(1-\prob{B_t})q_t/2$. Note that by our assumption 
of identifying all frequencies in $[1-1/n-1/K,1/n]$ with the upper border, the increasing steps are not cut at the upper border. We are left with bounding the probability of~$B_t$.

By our assumptions on the frequencies of lower index, we have that for all $t\le T$ that $\prob{B_t} \ge \left((1-2/n)^{i-1}\right)^2 \ge 1/e^4$ since it is sufficient to sample all positions of index less than~$i$ as~$0$ in both individuals. Hence, together with the probabilities of increasing and decreasing steps, we have for $X_t>0$ a drift of at least
\begin{equation}
\label{equation:drift}
 \expect{X_t-X_{t+1}\mid X_t} \ge \left(\frac{1}{2}+\frac{1}{2e^4}\right) q_t-\left(\frac{1}{2}-\frac{1}{2e^4}\right) q_t \ge \frac{q_t}{e^4},
\end{equation}
which, using that $q_t \ge (2/n)(1-1/n)\ge 1/n$, gives  
$\expect{X_t-X_{t+1}\mid X_t} \ge \frac{1}{e^4 n}.$
Adjusting by the possible negative effects of the frequency clamping mentioned in Section~\ref{section:drift}, the bound is still $\frac{1}{e^4 n}-K\frac{1}{(r-1)(n-1)^2}\ge \frac{1}{2e^4 n}\eqqcolon d$ since 
$K=o(n^2)$.  

This drift bound holds at any time before~$E_t$ happens. We will use a union bound to show that the probability of ever observing~$E_t$ 
in~$T$ steps is small enough.
Plugging in our parameters in Lemma~\ref{lemma:occupation}, we now have for all~$t\in [0,T]$ that 
\[\prob{X_t \ge b} \le 2Te^{\frac{1}{3 e^4 n (4/n)}(1-K/n)}\le 2Te^{1-\frac{1}{6e^4}c\ln n}.\]

The lemma now follows, noting that the actual starting time is~$t^*$.
\end{proofof}

The primary concept for the proof of Theorem~\ref{theorem:rcgacomplexity} is that the frequencies are likely to increase from their initial values of $1/r$. The effect of genetic drift is bounded if the update strength is selected small enough, which means that all frequencies corresponding to value~0 never fall below $1/(2r)$ with high probability. In this scenario, we demonstrate how the marginal probabilities have a tendency to shift toward their upper border, which increases the likelihood of finding the optimum. The following lemma, which also utilizes Lemma~\ref{lemma:rediscover-prefix} in its proof, establishes a positive trend towards optimal values for the \rcga.

\begin{lemma}
\label{lemma:sdrift}
 If 
$p^{(t)}_{j,0} \ge 1-2/n$ for all $j<i$ and $p^{(t)}_{i,0} \le 1-1/n-1/K$, then   \[\expect{\Delta_{i,0}\mid p^{(t)}_{i,0}} \geq \frac{p^{(t)}_{i,0}(1-p^{(t)}_{i,0})}{2e^4}\cdot \frac{1}{K}.\]
\end{lemma}

\begin{proof}
This proof is extracted from the proof of Lemma~\ref{lemma:rediscover-prefix}. From Equation~\ref{equation:drift}, we get 
\[\expect{X_t-X_{t+1}\mid X_t} \ge \frac{p^{(t)}_{i,0}(1-p^{(t)}_{i,0})}{e^4} \]
where $X_t = K(1-1/n - p^{(t)}_{i,0})$ is a stochastic process. Further, from the value of $X_t$, we can extract the following 
\[\expect{-p^{(t)}_{i,0} + p^{(t+1)}_{i,0}\mid p^{(t)}_{i,0}} \ge \frac{p^{(t)}_{i,0}(1-p^{(t)}_{i,0})}{e^4}\cdot \frac{1}{K}. \]
And, by $\Delta_{i,0} = p^{(t+1)}_{i,0}- p^{(t)}_{i,0}$, we get
\[\expect{\Delta_{i,0} \mid p^{(t)}_{i,0}} \ge \frac{p^{(t)}_{i,0}(1-p^{(t)}_{i,0})}{K e^4}.\]

Adjusting it by the expected decrease due to frequency clamping and 
using $K=o(n^2)$, we 
obtain 
the claimed 
\[\expect{\Delta_{i,0} \mid p^{(t)}_{i,0}} \ge \frac{p^{(t)}_{i,0}(1-p^{(t)}_{i,0})}{K e^4} - \frac{1}{(n-1)^2 (r-1)^2} \ge 
\frac{p^{(t)}_{i,0}(1-p^{(t)}_{i,0})}{2K e^4} .\quad\qed\]
\end{proof}

We are now ready to prove our main result.

\begin{proofof}{Theorem~\ref{theorem:rcgacomplexity}}
Recall that $p^{(t)}_{i,j}$ denote the marginal probabilities where $(i,j)\in \{1,\dots,n\}\times \{0,\dots,r-1\}$ at time $t$. We already defined the change of frequency in one step as $\Delta_{i,j}\coloneqq \Delta^{t}_{i,j}\coloneqq p^{(t+1)}_{i,j} - p^{(t)}_{i,j}$. We show that, starting with a setting where all frequencies are at least $1/r$, after $\bigo (nK\log n\log r\log K)$ iterations with probability $1-o(1)$ the global optimum has been found and no frequency corresponding to value~0 has dropped below $1/(2r)$. The basic idea is to use additive drift analysis with tail bounds in a series of phases to bound the expected optimization under the premise of low genetic drift, and then multiplicative drift analysis with tail bounds, which includes the above-mentioned concept of critical position and the analysis of the time until the critical position increases by~$1$.

To select a $K$ that makes genetic drift unlikely, we use Theorem~\ref{theorem:frequency-weak-preference}. We apply 
the theorem to show that in $T=c'nK\log n\log r\log K$ iterations no frequency corresponding to value~0 dropped below $1/(2r)$ with high probability where $c'$ is a given constant. For a single frequency, the probability is bounded by $\bigo (1/n^2)$ and by the union bound over all frequencies, the probability of at least one frequency dropping below $1/(2r)$ is still bounded by $\bigo (1/n)$:
\begin{align*}
\mathord{2\exp}\mathord{\left (-\frac{K^2}{8Tr^2}\right )} \leq \frac{1}{n^2} & \Leftrightarrow \mathord{2\exp}\mathord{\left (-\frac{K^2}{8c'nr^2K\log n\log r\log K}\right )} \leq \frac{1}{n^2}\\
& \Leftrightarrow {\left (-\frac{K}{8c'nr^2 \log n\log r\log K }\right )} \leq - 2\ln{n}.
\end{align*}
Hence, we choose $K\geq cnr^2 \log^2 n\log r$ where $c$ is a constant. The rest of the proof shows that the 
optimum is sampled in~$T$ iterations with high probability.

Let $m_t$ be the index of critical position at time~$t$, where we often drop the time index for notational convenience. Then, the main idea is to bound the time to increase $m$ by at least~$1$. So, this only bounds the time for a single frequency to reach its upper bound or, more precisely, for the frequency at the critical position. To bound the time for the frequency to reach its upper border (more precisely, recall that we identify this border with the interval $[1-1/n-1/K,1/n]$) with high probability, at first, we bound the time for $p^{(t)}_{m,0}$ to increase from $1/(2r)$ to at least $1/2$ and, after that, bound the time for $p^{(t)}_{m,0}$ to increase from $1/2$ to $1-1/n$.

The aim is to demonstrate that $\bigo (K \log r\log K)$ is an upper bound on the time needed for the critical position to increase. 
Using  Lemma~\ref{lemma:rediscover-prefix} and Corollary~\ref{corollary:positionbound}, we will analyze how the index~$m_t$ of the critical position increases. We note that 
the assumption of Lemma~\ref{lemma:rediscover-prefix} holds for~$T$ iterations with probability $1-o(1)$ if the constant~$c$ from~$K$ is chosen large enough.
Finally, the total time will be bounded by multiplying the time to increase the critical positions with the number of positions~$n$. 

Note that, although the frequency is $1/r$ after initialization of the algorithm, we pessimistically assume that it has dropped to $1/(2r)$ by the time that the index of the frequency is the critical position. Also, we have assumed above that genetic drift leads to a deviation of at most $1/(2r)$ from the expected value, so every frequency that corresponds to a neutral position does not drop below $1/r - 1/(2r)$ in the number $T$ of steps considered above. Now, we split the time for the frequency to increase from $1/(2r)$ to $1/2$ into phases. To reach the value $1/2$ starting from $1/(2r)$, we need $r/2$ phases with an increase of $1/r$ per phase, starting from phase~2, and for phase~1, we need an increase of $3/(2r)$. We consider phase indices $k= 1, 2, \dots, r/2$.

Further, we assume that for the starting time $T_{k}$ of phase $k$ it holds $p^{(T_{k})}_{m,0}\ge k/r$ for each $m\in \{1,\dots,n\}$.
After that, by applying the additive drift  theorem with tail bounds~\cite[Theorem 2]{Timo2016Algo}, we bound with high probability  the time $T_k$ to conclude phase~$k$ by at most
\begin{align*}
\prob{T_{k}\ge s_k} & \le \exp\left( -\frac{s_k \varepsilon^2}{8c^2} \right)
\end{align*}
where $c$ is a bound on the step size, $\varepsilon$ is a bound on the drift, and $s_k$ must satisfy $s_k\ge 2D/\varepsilon$ where $D$ is the distance that the process should bridge. Now, in our case, $c=1/K$ since that is the maximum change of a frequency at a time, $\varepsilon = (k/r)\cdot(1/K)\cdot(1/(2e^4))$ as already derived in Lemma~\ref{lemma:sdrift}. In general, the distance to be bridged $D\le (k+1)/r+1/K-k/r\le 2/r$ for $k\ge 2$ and for $k=1$, we have $D\le 2/r+1/K-1/(2r)\le 3/r$. So, we select $D=3/r$ in every phase and use $s_k=2D/\varepsilon$.
 (Since the frequency value $k/r$ may not be achievable exactly, we add up to~$1/K$ to hit the smallest possible frequency above~$k/r$.)

Putting all together, we get
\begin{align*}
\prob{T_{k}\ge s_k} \le \exp\left( -\frac{2}{\varepsilon}\cdot\frac{3}{r}\cdot\frac{K^2\varepsilon^2}{8} \right)  \le \exp\left( -\frac{3K^2}{8r}\cdot\frac{k}{rKe^4} \right) \le \exp\left( -\frac{3Kk}{8r^2e^4} \right).
\end{align*}

Now plugging in our assumption $K\geq cnr^2 \log^2 n \log r$, we have
\begin{align*}
\prob{T_{k}\ge s_k} \le \exp\left( -\frac{3cnkr^2 \log^2 n\log r}{8r^2e^4} \right) \le \exp\left( -\frac{3cnk \log^2 n\log r}{8e^4} \right) = n^{-\omega(1)}.
\end{align*}
%

Now, we take a union bound over all $r/2$ phases and still have high probability of every phase finishing within at most $2D/\varepsilon = 6e^4 K/k$ steps. Further, we sum up all the $s_k$ to bound the total time that we consider for all phases $1,\dots,r/2$ is $\bigo(K\log r)$ with probability at least $1-rn^{-\omega(1)}$. Since $r=\mathit{poly}(n)$, this failure  probability is still $o(1)$.

After the frequency has reached at least $1/2$, we need to analyze the remaining time until $p^{(t)}_{m,0}$ attains its maximum at $1-1/n$. Here, we assume (according to Corollary~\ref{corollary:positionbound}) all frequencies that have reached the upper border $1-1/n$ before to be bounded from below by $1-2/n$. In this phase, we apply the multiplicative drift with tail bounds~\cite[Theorem 2.4.5]{Lengler2020}. We have $p^{(t)}_{m,0}\ge  1/2$ as starting point. Let $q^{(t)}_{i,j}\coloneqq 1 -1/n - p^{(t)}_{i,j}$ where $(i,j)\in \{1,\dots,n\}\times \{0,\dots,r-1\}$ at time $t$. Further, we bound $p^{(t)}_{m,0} (1-p^{(t)}_{m,0})$ by using $p^{(t)}_{m,0}\ge  1/2$ and $1 - p^{(t)}_{m,0} = q^{(t)}_{m,0} + 1/n$, then by Lemma~\ref{lemma:sdrift}, we get 
\begin{align*}
\expect{q^{(t)}_{m,0} - q^{(t+1)}_{m,0}}  \geq \frac{2p^{(t)}_{m,0}(1-p^{(t)}_{m,0})}{2K e^4} 
 \geq \frac{2\cdot\frac{1}{2}\cdot (q^{(t)}_{m,0}+\frac{1}{n})}{2K e^4}
 \geq \frac{(q^{(t)}_{m,0}+\frac{1}{n})}{2K e^4} \geq \frac{q^{(t)}_{m,0}}{2K}.
\end{align*}

According to our model that identifies the upper frequency border and all values at least $1-1/n-1/K$, the smallest possible state is $1/K$. So, here we apply the multiplicative drift theorem \cite[Theorem 2.4.5]{Lengler2020} with $x_{\min}=1/K$, $X_t=q^{(t)}_{m,0}$ and already derived $\delta=1/(2K)$. Then, choosing $r'>0$, it holds for $T\coloneqq \min\{t\mid X_t = 0\}$ that 
\begin{equation*}
    \prob{T > (r'+\ln(X_0/x_{\min}))/\delta} \le e^{-r'}.
\end{equation*}

By selecting $r'=c\ln{K}$, where $c$ is a constant, we get 
$\prob{T> (2c+1)K \ln{K}} \le e^{-c\ln{K}}$. So, the remaining time for the frequency of the critical position to reach at least $1-1/n-1/K$ is $\bigo (K\log K)$ with probability $1-o(1)$.

By adding the two stages, we have that the total time spent for a position is $\bigo (K\log r\log K)$. Further, summing over all positions $m\in \{1,\dots,n\}$, we obtain the total time until all frequencies for value~$0$ have been raised to at least $1-1/n-1/K$ at least once is $\bigo (n K \log r\log K)$ with high probability. 

Note that, the \rcga creates the optimum of \rlo during the subsequent iteration with probability at least $(1-\frac{2}{n})^n\ge \frac{1}{e^2}$ once for every $i\in\{1,\dots,n\}$ it holds that $p^{(t)}_{i,0}\ge 1-2/n$. The probability of not creating the optimum within the next $\log n$ iterations is at most $(1-\frac{1}{e^2})^{\log n} =  o(1)$. Now, the total time to sample the optimum is $\bigo (n K \log r\log K + \log n)$ $=\bigo (n K \log r\log K)$ with high probability.   

Using the above  choice $K=c n r^2 \log^2 n\log r$ that prevents any frequency dropping below $1/(2r)$ with high probability, the runtime of the \rcga on \rlo is $\bigo(n^2 r^2 \log^2 n \log^2 r (\log n + \log r + \log\log n + \log\log r)) = \bigo(n^2 r^2 \log^2 n \log^2 r (\log n + \log r )) = \bigo(n^2 r^2 \log^3 n \log^2 r)$, noting that  $r=poly(n)$ implies $\log r=\bigo(\log n)$, and this holds with high probability.
\end{proofof}

\section{Experiments}
\label{section:experiments}

In this section, we present the results of the experiments we performed to evaluate the performance of the proposed algorithm with border restrictions. We theoretically prove the expected runtime for the \rcga on \rlo. We implemented the algorithm using the $C$ programming language, using the \text{WELL1024a} random number generator.

We conducted the \rcga on \rlo using two distinct aspects in our experiment. First, as presented in Fig.~\ref{figure:rcGAn500-1000}, we provide the average number of iterations for a variety of hypothetical population sizes ($K \in\{100,\dots,1000\}$). We next compare the outcomes for various $r\in\{2,\dots,10\}$. This empirical study allows us to clearly see how the runtime depends on~$r$ by comparing the findings for various $r$. Furthermore, we can figure out the value of $K$ at which the minimum of runtime is met. In every instance, we note that it begins at a high value, drops to a minimum, and then rises once more throughout the remainder of $K$. For instance, we find that the minimum $K$ is approximately 168 (Fig.~\ref{figure:rcGAn500-1000}: $n=500$ and $r=6$).

\begin{figure}[t]
	\centering
\begin{tikzpicture}[scale=0.50]
\begin{axis}[
    xlabel={Population Size (K)},
    ylabel={Avg. Number of Iteration},
    xmin=100, xmax=1000,
    ymin=10000, ymax=10000000,
    legend pos=outer north east,
    legend cell align=left,
    ymajorgrids=true,
    grid style=dashed,
]
 
   \addplot[
    color=olive,
    mark=dot,
    ]
    table[ignore chars={(,)},col sep=comma] {Data/LO-n500-r2.txt};
    
    \addplot[
    color=red,
    mark=dot,
    ]
    table[ignore chars={(,)},col sep=comma] {Data/LO-n500-r3.txt}; 
    
    \addplot[
    color=green,
    mark=dot,
    ]
    table[ignore chars={(,)},col sep=comma] {Data/LO-n500-r4.txt}; 
    
    \addplot[
    color=black,
    mark=dot,
    ]
    table[ignore chars={(,)},col sep=comma] {Data/LO-n500-r5.txt}; 
    
    \addplot[
    color=blue,
    mark=dot,
    ]
    table[ignore chars={(,)},col sep=comma] {Data/LO-n500-r6.txt}; 
    
    \addplot[
    color=magenta,
    mark=dot,
    ]
    table[ignore chars={(,)},col sep=comma] {Data/LO-n500-r7.txt}; 
    
    \addplot[
    color=brown,
    mark=dot,
    ]
    table[ignore chars={(,)},col sep=comma] {Data/LO-n500-r8.txt}; 
    
    \addplot[
    color=teal,
    mark=dot,
    ]
    table[ignore chars={(,)},col sep=comma] {Data/LO-n500-r9.txt}; 
    
    \addplot[
    color=violet,
    mark=dot,
    ]
    table[ignore chars={(,)},col sep=comma] {Data/LO-n500-r10.txt}; 
\end{axis}
\end{tikzpicture}
\begin{tikzpicture}[scale=0.50]
\begin{axis}[
    xlabel={Population Size (K)},
    ylabel={Avg. Number of Iteration},
    xmin=100, xmax=1000,
    ymin=10000, ymax=30000000,
    legend pos=outer north east,
    legend cell align=left,
    ymajorgrids=true,
    grid style=dashed,
]
 
   \addplot[
    color=olive,
    mark=dot,
    ]
    table[ignore chars={(,)},col sep=comma] {Data/LO-n1000-r2.txt};
    
    \addplot[
    color=red,
    mark=dot,
    ]
    table[ignore chars={(,)},col sep=comma] {Data/LO-n1000-r3.txt}; 
    
    \addplot[
    color=green,
    mark=dot,
        ]
    table[ignore chars={(,)},col sep=comma] {Data/LO-n1000-r4.txt}; 
    
    \addplot[
    color=black,
    mark=dot,
    ]
    table[ignore chars={(,)},col sep=comma] {Data/LO-n1000-r5.txt}; 
    
    \addplot[
    color=blue,
    mark=dot,
    ]
    table[ignore chars={(,)},col sep=comma] {Data/LO-n1000-r6.txt}; 
    
    \addplot[
    color=magenta,
    mark=dot,
    ]
    table[ignore chars={(,)},col sep=comma] {Data/LO-n1000-r7.txt}; 
    
    \addplot[
    color=brown,
    mark=dot,
    ]
    table[ignore chars={(,)},col sep=comma] {Data/LO-n1000-r8.txt}; 
    
    \addplot[
    color=teal,
    mark=dot,
    ]
    table[ignore chars={(,)},col sep=comma] {Data/LO-n1000-r9.txt}; 
    
    \addplot[
    color=violet,
    mark=dot,
    ]
    table[ignore chars={(,)},col sep=comma] {Data/LO-n1000-r10.txt};

\legend{$r=2$, $r=3$, $r=4$, $r=5$, $r=6$, $r=7$, $r=8$, $r=9$, $r=10$}    
\end{axis}
\end{tikzpicture}
\caption{\textmd{Empirical runtime of the \rcga on \rlo; for $n=500$ (left-hand side) and $n=1000$ (right-hand side), $K\in\{100,\dots,1000\}$ and averaged over 1000 runs.}}
\label{figure:rcGAn500-1000}
\end{figure}

In the other aspects, we compare the results for different $n\in\{100,\dots,400\}$ and offer the average number of iterations for a range of hypothetical population sizes ($K \in\{100,\dots,1000\}$). We showed the various plots for $r\in\{2,\dots,6\}$ in Fig.~\ref{figure:rcGAn100-400}. Here, we see an identical situation as in the preceding one. The empirical runtime begins at a very high number in each case, decreases to a minimum, and then increases once again for the remaining $K$. By comparing the outcomes, this empirical study makes it evident how the runtime depends on~$r$. We can conclude that the theoretical analysis's bound is probably not tight based on the experimental configuration.

\begin{figure}[t]
	\centering
\begin{tikzpicture}[scale=0.4]
\begin{axis}[
    xlabel={Population Size (K)},
    ylabel={Avg. Number of Iteration},
    xmin=100, xmax=1000,
    ymin=10000, ymax=1000000,
    legend pos=outer north east,
    legend cell align=left,
    ymajorgrids=true,
    grid style=dashed,
]
   \addplot[
    color=red,
    mark=dot,
    ]
    table[ignore chars={(,)},col sep=comma] {Data/LO-n100-r2.txt};
    
    \addplot[
    color=blue,
    mark=dot,
    ]
    table[ignore chars={(,)},col sep=comma] {Data/LO-n200-r2.txt};
    
    \addplot[
    color=green,
    mark=dot,
    ]
    table[ignore chars={(,)},col sep=comma] {Data/LO-n300-r2.txt};
    
     \addplot[
    color=black,
    mark=dot,
    ]
    table[ignore chars={(,)},col sep=comma] {Data/LO-n400-r2.txt};
\end{axis}
\end{tikzpicture}
\begin{tikzpicture}[scale=0.4]
\begin{axis}[
    xlabel={Population Size (K)},
    ylabel={Avg. Number of Iteration},
    xmin=100, xmax=1000,
    ymin=10000, ymax=2000000,
    legend pos=outer north east,
    legend cell align=left,
    ymajorgrids=true,
    grid style=dashed,
]
 
   \addplot[
    color=red,
    mark=dot,
    ]
    table[ignore chars={(,)},col sep=comma] {Data/LO-n100-r3.txt};
    
     \addplot[
    color=blue,
    mark=dot,
    ]
    table[ignore chars={(,)},col sep=comma] {Data/LO-n200-r3.txt};
    
    \addplot[
    color=green,
    mark=dot,
    ]
    table[ignore chars={(,)},col sep=comma] {Data/LO-n300-r3.txt};
    
     \addplot[
    color=black,
    mark=dot,
    ]
    table[ignore chars={(,)},col sep=comma] {Data/LO-n400-r3.txt};
\end{axis}
\end{tikzpicture}
\begin{tikzpicture}[scale=0.4]
\begin{axis}[
    xlabel={Population Size (K)},
    ylabel={Avg. Number of Iteration},
    xmin=100, xmax=1000,
    ymin=10000, ymax=3000000,
    legend pos=outer north east,
    legend cell align=left,
    ymajorgrids=true,
    grid style=dashed,
]
   \addplot[
    color=red,
    mark=dot,
    ]
    table[ignore chars={(,)},col sep=comma] {Data/LO-n100-r4.txt};
    
   \addplot[
    color=blue,
    mark=dot,
    ]
    table[ignore chars={(,)},col sep=comma] {Data/LO-n200-r4.txt};
    
    \addplot[
    color=green,
    mark=dot,
    ]
    table[ignore chars={(,)},col sep=comma] {Data/LO-n300-r4.txt};
    
    \addplot[
   color=black,
   mark=dot,
   ]
   table[ignore chars={(,)},col sep=comma] {Data/LO-n400-r4.txt};

\legend{$n=100$, $n=200$, $n=300$, $n=400$}    
\end{axis}
\end{tikzpicture}
\begin{tikzpicture}[scale=0.4]
\begin{axis}[
    xlabel={Population Size (K)},
    ylabel={Avg. Number of Iteration},
    xmin=100, xmax=1000,
    ymin=10000, ymax=4000000,
    legend pos=outer north east,
    legend cell align=left,
    ymajorgrids=true,
    grid style=dashed,
]
   \addplot[
    color=red,
    mark=dot,
    ]
    table[ignore chars={(,)},col sep=comma] {Data/LO-n100-r5.txt};
    
    \addplot[
    color=blue,
    mark=dot,
    ]
    table[ignore chars={(,)},col sep=comma] {Data/LO-n200-r5.txt};
    
    \addplot[
    color=green,
    mark=dot,
    ]
    table[ignore chars={(,)},col sep=comma] {Data/LO-n300-r5.txt};
    
    \addplot[
   color=black,
   mark=dot,
   ]
   table[ignore chars={(,)},col sep=comma] {Data/LO-n400-r5.txt};
 
\end{axis}
\end{tikzpicture}
\begin{tikzpicture}[scale=0.4]
\begin{axis}[
    xlabel={Population Size (K)},
    ylabel={Avg. Number of Iteration},
    xmin=100, xmax=1000,
    ymin=10000, ymax=6000000,
    legend pos=outer north east,
    legend cell align=left,
    ymajorgrids=true,
    grid style=dashed,
]
 
   \addplot[
    color=red,
    mark=dot,
    ]
    table[ignore chars={(,)},col sep=comma] {Data/LO-n100-r6.txt};
    
     \addplot[
    color=blue,
    mark=dot,
    ]
    table[ignore chars={(,)},col sep=comma] {Data/LO-n200-r6.txt};
    
    \addplot[
    color=green,
    mark=dot,
    ]
    table[ignore chars={(,)},col sep=comma] {Data/LO-n300-r6.txt};
   
    \addplot[
   color=black,
   mark=dot,
   ]
   table[ignore chars={(,)},col sep=comma] {Data/LO-n400-r6.txt};
     
\end{axis}
\end{tikzpicture}
\begin{tikzpicture}[scale=0.4]
\begin{axis}[
    xlabel={Population Size (K)},
    ylabel={Avg. Number of Iteration},
    xmin=100, xmax=1000,
    ymin=10000, ymax=8000000,
    legend pos=outer north east,
    legend cell align=left,
    ymajorgrids=true,
    grid style=dashed,
]

   \addplot[
    color=red,
    mark=dot,
    ]
    table[ignore chars={(,)},col sep=comma] {Data/LO-n100-r7.txt};
    
   \addplot[
    color=blue,
    mark=dot,
    ]
    table[ignore chars={(,)},col sep=comma] {Data/LO-n200-r7.txt};
    
   \addplot[
   color=green,
   mark=dot,
   ]
   table[ignore chars={(,)},col sep=comma] {Data/LO-n300-r7.txt};
    
    \addplot[
   color=black,
   mark=dot,
   ]
   table[ignore chars={(,)},col sep=comma] {Data/LO-n400-r7.txt};
    
\legend{$n=100$, $n=200$, $n=300$, $n=400$}    
\end{axis}
\end{tikzpicture}
\caption{\textmd{Empirical runtime of the \rcga on \rlo; for $r=2$ (top-left), $r=3$ (top-middle), $r=4$ (top-right), $r=5$ (bottom-left), $r=6$ (bottom-middle) and $r=7$ (bottom-right); $n\in\{100,\dots,400\}$, $K\in\{100,\dots,1000\}$ and averaged over 200 runs.}}
\label{figure:rcGAn100-400}
\vspace{-2.0em}
\end{figure}

\section{Conclusion}
\label{section:comclusion}

We have conducted a runtime analysis of a multi-valued \cga on a generalized \lo function with border restrictions and bounded its  runtime with high probability. For constant~$r$, our bound is only by polylogarithmic factors larger than the typical $\Theta(n^2)$ runtime that many 
randomized search heuristics exhibit on this problem. To prove the main results, we applied additive drift analysis with tail bounds and then multiplicative drift with tail bounds to 
analyze the growth of relevant frequencies and overall progress in the so-called critical positions. Additionally, we used occupation probabilities methods.  We believe that our runtime bounds for the \rcga on \rlo can be improved based on the experimental results.

In this work, we have used the function \rlo, which represents categorical values and indicates that only the appropriate value for a position can add to the fitness. The next challenge is to examine the \rcga on multi-valued functions where each position can contribute 
more than two values to the fitness.

\begin{credits}
\subsubsection{\ackname} 
This work has been supported by the Danish Council for Independent Research through grant 10.46540/2032-00101B.
\end{credits}

%
\bibliographystyle{splncs04}
\bibliography{References}

\end{document}